\documentclass{article}

\usepackage{microtype}
\usepackage{graphicx}
\usepackage{subfigure}
\usepackage{bbm}
\usepackage{booktabs} 
\usepackage{epstopdf}
\usepackage {tikz}
\usetikzlibrary{calc, positioning, decorations.pathreplacing}
\usepackage{bm}
\usepackage{todonotes}

\usepackage{appendix}
\usepackage{amsmath,amssymb,amsfonts,dsfont,mathrsfs,environ}
\usepackage{amsthm}
\usepackage{algorithm}
\usepackage{algorithmic}
\usepackage{authblk}
\usepackage{titlesec}
\usepackage[a4paper, total={6.2in, 8.5in}]{geometry}
\usepackage{fancyhdr,lipsum}
\usepackage{hyperref}
\usepackage[square,sort,comma,numbers]{natbib}
\usepackage{times}

\title{Learning Binary Latent Variable Models: A Tensor Eigenpair Approach}
\date{}
\author
{
	Ariel Jaffe$^1 \footnote{Email addresses:  \url{ariel.jaffe@weizmann.ac.il},  \url{roi.weiss@weizmann.ac.il}, \url{boaz.nadler@weizmann.ac.il}, \url{shai.carmi@huji.ac.il}, \url{yuval.kluger@yale.edu}}$, Roi Weiss$^1$, 
	Shai Carmi$^2$, Yuval Kluger$^{3,4,5}$ and Boaz Nadler$^1$
	\vspace{0.2cm} \\
	$^1${\footnotesize Dept. of Computer Science and Applied Mathematics, Weizmann Institute of Science, Rehovot 7610001, Israel }\\
	$^2${\footnotesize Braun School of Public Health and Community Medicine, The Hebrew University of Jerusalem, Jerusalem 9112102, Israel}\\
	$^3${\footnotesize Interdepartmental Program in Computational Biology and Bioinformatics, Yale University, New Haven, CT 06511, USA}\\
	$^4${\footnotesize Dept. of Pathology and Yale Cancer Center, Yale University School of Medicine, New Haven, CT 06520, USA}\\
	$^5${\footnotesize Program of Applied Mathematics, Yale University , New Haven, CT 06511, USA}
}

\newcommand{\xv}{\bm x}
\newcommand{\bv}{\bm b}

\newcommand{\fv}{\bm x}
\newcommand{\tv}{\bm t}
\newcommand{\hv}{\bm h}
\newcommand{\vv}{\bm{v}}
\newcommand{\yv}{\bm y}

\newcommand{\wv}{\bm w}
\newcommand{\R}{\mathbb{R}}
\newcommand{\T}{\mathcal T}

\newcommand{\E}{\mathbb E}

\newcommand{\nrm}[1]{{\|#1\|}}
\newcommand{\eps}{\varepsilon}

\newcommand{\g}{\gamma}
\newcommand{\gv}{\bm g}

\newcommand{\rank}{\operatorname{rank}}
\newcommand{\diag}{\operatorname{diag}}
\newcommand{\noise}{\xi}
\newcommand{\noisev}{\bm \noise}
\newcommand{\vast}{\vv^*}
\newcommand{\uv}{\bm u}

\newcommand{\hvv}{\hat{\vv}}

\newcommand{\ffo}{\bm\mu}
\newcommand{\effo}{\hat\ffo}
\newcommand{\fso}{M}
\newcommand{\efso}{\hat\fso}
\newcommand{\hfo}{\bm p}
\newcommand{\hso}{C}
\newcommand{\M}{\mathcal{M}}
\newcommand{\fto}{\M}
\newcommand{\efto}{\hat{\fto}}
\newcommand{\hto}{\mathcal C}
\newcommand{\wto}{\mathcal{W}}
\newcommand{\twto}{\tilde{\mathcal{W}}}
\newcommand{\ewto}{\hat{\mathcal{W}}}
\newcommand{\wm}{K}

\renewcommand{\T}{\mathcal T}

\newcommand{\beq}{\begin{eqnarray*}}
	\newcommand{\eeq}{\end{eqnarray*}}
\newcommand{\beqn}{\begin{eqnarray}}
\newcommand{\eeqn}{\end{eqnarray}}

\newcommand{\Winv}{W^\dagger}
\newcommand{\Wemp}{\hat W}

\newcommand{\EVset}{V}
\newcommand{\Vsig}{\EVset_\sigma}
\newcommand{\eVsig}{\hat{V}_\sigma}
\newcommand{\Vsub}{\bar V}
\newcommand{\Vsubsig}{\bar \EVset_\sigma}
\newcommand{\Vast}{\Winv}
\newcommand{\Ksig}{\wm_\sigma}
\newcommand{\eKsig}{{\hat \wm}_\sigma}

\newcommand{\eUsig}{\hat{U}_\sigma}
\newcommand{\lamths}{\tau}
\newcommand{\colvec}[2]{({#1},{#2})}
\newcommand{\argmin}{\operatorname*{argmin}}

\newcommand{\spn}{\operatorname{span}}
\newcommand{\proj}{L}
\newcommand{\score}{\Delta}
\newcommand{\Uast}{U^*}
\newcommand{\sphere}{\mathbb S}
\newcommand{\samp}{X}
\newcommand{\wWinv}{Y^*}
\newcommand{\hess}{J_p}
\newcommand{\tuv}{\tilde{\uv}}
\newcommand{\tvv}{\tilde{\vv}}
\newcommand{\ttv}{\tilde{\tv}}
\newcommand{\twv}{\tilde{\wv}}

\newcommand{\tlambda}{\tilde{\lambda}}

\newcommand{\Com}{F}
\newcommand{\eCom}{\hat{\Com}}
\newcommand{\chr}{\mathbbm{1}}
\newcommand{\expbr}{r}
\NewEnviron{salign}{%
	\begin{align}\begin{split}
	\BODY
	\end{split}\end{align}
}

\newtheorem{theorem}{Theorem}
\newtheorem{lemma}{Lemma}
\newtheorem{proposition}{Proposition}
\newtheorem{definition}{Definition}
\newtheorem*{definition*}{Definition}

\newtheorem*{remark*}{Remark}

\begin{document}
\maketitle








\begin{abstract}
Latent variable models with hidden binary units appear in various applications.
Learning such models, in particular in the presence of noise, is a challenging computational problem. In this paper we propose a novel spectral approach to this problem, based on the eigenvectors of both the second order moment matrix and third order moment tensor of the observed data.
We prove that under mild non-degeneracy conditions, our method consistently estimates the model parameters at the optimal parametric rate.
Our 
tensor-based 
method generalizes previous orthogonal tensor decomposition approaches, where the hidden units were assumed to be either statistically independent or mutually exclusive.
We illustrate the consistency of our method on simulated data and  
demonstrate its usefulness in learning a common model for population
mixtures in genetics.
\end{abstract}

\section{Introduction}
\label{sec:introduction}
In this paper we propose a spectral method for learning the following binary latent variable model, shown in Figure \ref{Fig:graph}.
The hidden layer, $\hv = (h_1,\dots,h_d)$, consists of $d$ binary random variables with an unknown joint distribution $P_{\hv}:\{0,1\}^d\to[0,1]$.
The observed vector $\fv \in \R^m$ of $m\geq d$ features is modeled as
\begin{equation}
\fv = W^\top \hv + \sigma\noisev,
\label{eq:model}
\end{equation}
where $W \in \R^{d \times m}$ is an unknown {weight} matrix assumed to be full rank $d$.
Here, $\sigma\geq 0$ is the noise level and $\noisev$ 
is an additive noise vector independent of $\hv$, 
whose $m$ coordinates are all i.i.d. zero mean and unit variance random variables.
For simplicity we assume it is Gaussian, though our method can be modified to handle other noise distributions.

The model in \eqref{eq:model} appears, for example, in overlapping clustering \citep{banerjee2005model,baadel2016overlapping}, in various problems in bioinformatics \citep{segal2002decomposing,becker2011multifunctional,slawski2013matrix}, and in blind source separation \citep{van1997analytical}.
A special instance of model \eqref{eq:model} is the Gaussian-Bernoulli restricted Boltzmann machine (G-RBM) where the distribution $P_{\hv}$ is further assumed to have a parametric energy-based structure \citep{hinton2006reducing,cho2011improved,wang2012analysis}.
G-RBMs were used, e.g., in modeling human motion \citep{taylor2007modeling} and natural image patches \citep{melchior2017gaussian}.

Given $n$ i.i.d.\ samples $\fv_1,\dots, \fv_n$ from model \eqref{eq:model}, the goal is to estimate the weight matrix $W$.
A common approach for learning $W$ is by maximum likelihood.
As this function is non-convex, common optimization schemes include the EM algorithm and alternating least squares (ALS).
In addition, several works developed iterative methods specialized to G-RBMs \citep{hinton2010practical,cho2011improved}.
All these methods, however, often lack consistency guarantees and may not be well suited for large datasets due to their potential slow convergence.
This is not surprising, as learning $W$ under model \eqref{eq:model} is believed to be computationally hard; see for example \citet{mossel2005learning}.

Over the past years, several works considered variants and specific instances of model \eqref{eq:model} under additional assumptions on the distribution $P_{\hv}$ or on the weight matrix $W$.
For example, when $P_{\hv}$ is a product distribution, the learning problem becomes that of independent component analysis (ICA) with binary signals \citep{hyvarinen2004independent}.
In this case, several methods have been derived for estimating $W$ and under suitable non-degeneracy conditions were proven to be both computationally efficient and statistically consistent \citep{shalvi1993super,frieze1996learning,regalia2003monotonic,hyvarinen2004independent,anandkumar2014tensor,jain2014learning}.
Similarly, when the hidden units are mutually exclusive, namely $P_{\hv}$ has support $\hv\in\{\bm e_i\}_{i=1}^d$, the model is a Gaussian mixture (GMM) with $d$ spherical components with linearly independent means.
Efficient and consistent algorithms have been derived for this case as well \citep{moitra2010settling,anandkumar2012spectral,anandkumar2012method,hsu2013learning}.
Among those, most relevant to this work are orthogonal tensor decomposition methods \citep{anandkumar2014tensor}.
Interestingly, 
these methods can learn some additional latent models, with hidden units that are not necessarily binary, such as Dirichlet allocation and other correlated topic models \citep{arabshahi2017spectral}.

Learning 
$W$ given the observed data $\{\fv_j\}_{j=1}^n$ can also be viewed as a noisy \emph{matrix factorization} problem.
If $W$ is known to be non-negative, then various non-negative matrix factorization methods can be used.
Moreover, under appropriate conditions, some of these methods were proven to be  computationally efficient and consistent \citep{donoho2004does,arora2012computing}.
For general full rank $W$, the matrix factorization method in \citet{slawski2013matrix} (SHL) exactly recovers $W$ when $\sigma=0$
with a runtime exponential in $d$.
This method, however, can handle only low levels of noise and 
has no consistency guarantees when $\sigma>0$.

\paragraph{A tensor eigenpair approach}
In this paper we propose a novel spectral method for learning $W$ 
which is based on the eigenvectors of both the second order moment matrix and the third order moment tensor of the observed data.
We prove that our method is consistent under mild non-degeneracy conditions and
achieves the parametric rate $O_P(n^{-\frac{1}{2}})$ for any noise level $\sigma\geq 0$.

The non-degeneracy conditions we pose are significantly weaker than those required by 
previous tensor decomposition methods mentioned above. 
In particular, their assumptions and resulting methods can be viewed as specific cases of our more general approach. 

Similarly to the matrix factorization method in \citet{slawski2013matrix}, our algorithm has runtime linear in $n$, polynomial in $m$, and in general exponential in $d$.
With our current Matlab implementation, most of the runtime is spent on computing the eigenpairs of a $d\times d\times d$ tensor.
Practically, our method, implemented without any particular optimization, can learn a model with 12 hidden units in less than ten minutes on a standard PC.
Furthermore, the overall runtime can be significantly reduced, since the step of computing the tensor eigenpairs can be embarrassingly parallelized.

\begin{figure}[t]
\begin{center}
\resizebox{0.44\linewidth}{!}{
{\begin{tikzpicture}[
	scale=.6, -latex ,auto ,node distance =3 cm and 6cm,on grid, semithick,
				state/.style ={ circle ,top color =white , bottom color = white ,
					draw, black , text=black , minimum width =.7 cm, inner sep=0pt}]
				\node[state] (h1) at (1,3) {$h_1$};
				\node[state] (h2) at (2.6,3) {$h_2$};
				\node[state] (h3) at (5,3) {$h_d$};
				\node[state] (F1) at (0,0) {$x_1$};
				\node[state] (Fi) at (3,0) {$x_i$};		
				\node[state] (Fm) at (6,0) {$x_m$};		
				\path (h1) edge node [above =0.15 cm,left = 0.1cm] {$W_{11}$}(F1);
				\path (h2) edge node [above =0.15 cm,left = 0.15cm] {}(F1);
				\path (h3) edge node [above =0.15 cm,left = 0.35cm] {}(F1);
				\path (h1) edge node [above =0.15 cm,left = 0.15cm] {}(Fi);
				\path (h2) edge node [above =0.15 cm,left = 0.15cm] {}(Fi);
				\path (h3) edge node [above =0.15 cm,left = 0.15cm] {}(Fi);
				\path (h1) edge node [above =0.15 cm,left = 0.15cm] {}(Fm);
				\path (h2) edge node [above =0.15 cm,left = 0.15cm] {}(Fm);
				\path (h3) edge node [above =0.15 cm,right = 0.1cm] {$W_{dm}$}(Fm);
				\path (F1) -- node[auto=false]{\ldots} (Fi);
				\path (h2) -- node[auto=false]{\ldots} (h3);
				\path (Fi) -- node[auto=false]{\ldots} (Fm);
				\node at (0,4.5) {};
\end{tikzpicture}
}	
}
\caption{The binary latent variable model.}
\label{Fig:graph}
\end{center}
\end{figure}
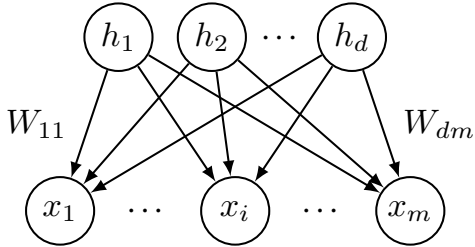

\paragraph{Paper outline} 
In the next section we fix the notation and provide necessary background on tensor eigenpairs.
In Section \ref{sec:method} we introduce our method in the case $\sigma=0$.
The case $\sigma\geq0$ is treated in Section \ref{sec:method_noise}.
Experiments with our method and comparison to other approaches appear in 
Section \ref{sec:experiments}.
All proofs are deferred to the appendices.

\section{Preliminaries}
\label{sec:tensor preliminaries}
\paragraph{Notation}
We abbreviate $[d]=\{1,\dots,d\}$ and denote $\bm e_i$ as the $i$-th unit vector with entries $(e_{i})_j=\delta_{ij}$.
We slightly abuse notation and view a matrix $W$ also as the set of its columns, namely $\wv\in W$ is some column of $W$ and $\spn(W)$ is the span of all its columns.
The unit sphere is denoted by $\sphere_{d-1} = \{\uv\in\R^d: \nrm{\uv}=1\}$.

A tensor $\mathcal T \in \R^{d \times d \times d}$ is symmetric if $\T_{ijk}= \T_{\pi(i,j,k)}$ for all permutations $\pi$ of $i,j,k$.
Here, we consider only symmetric tensors.
$\T$ can also be seen as a {\em multi-linear} operator:
for matrices $W^{1},W^{2},W^{3}$ with $W^{i}\in\R^{d \times d_i}$, the tensor-mode product, denoted $\T(W^{1},W^{2},W^{3})$, is a $d_1\times d_2\times d_3$ tensor whose $(i_1,i_2,i_3)$-th entry is
$$
\sum_{j_1,j_2,j_3\in[d]} 
W^{1}_{j_1 i_1} W^{2}_{j_2 i_2} W^{3}_{j_3 i_3} \T_{j_1 j_2 j_3}.
$$

\paragraph{Tensor eigenpairs}
Several types of eigenpairs of a tensor have been proposed.
Here, we consider the following definition, termed $Z$-eigenpairs by \citet{qi2005eigenvalues} and $l_2$-eigenpairs by \citet{lim2005singular}.
Henceforth we just call them eigenpairs.

\begin{definition}
\label{def:eig_def}
$(\uv,\lambda)\in\R^{d}\times\R$ is an eigenpair of $\T$ if
\beqn
\label{eq:fixed_point_def}
\T(I,\uv,\uv) = \lambda \uv \quad \text{and} \quad \|\uv\|=1.
\eeqn
\end{definition}

Note that if $(\uv,\lambda)$ is an eigenpair then the eigenvalue is simply $\lambda=\T(\uv,\uv,\uv)$.
In addition, $(-\uv,-\lambda)$ is also an eigenpair. Following common practice, we treat these two pairs as one.
So, without loss of generality, we make the convention that $\lambda\geq 0$.

In contrast to the matrix case, the number of eigenvalues $\{\lambda\}$ of a tensor $\T\in\R^{d\times d\times d}$ can be much larger than $d$.
As shown by \citet{cartwright2013number}, for a $d\times d\times d$ tensor, there can be at most $2^d-1$ of them.
With precise definitions appearing in \citet{cartwright2013number}, for a 
\emph{generic} tensor, all its eigenvalues have multiplicity one and the number of eigenpairs $\{(\uv,\lambda)\}$ is at most $2^d-1$.

In principle, computing the set of all eigenpairs of a general symmetric tensor is 
a \#P problem 
\citep{hillar2013most}.
Nevertheless, several methods have been proposed for computing at least {some} eigenpairs, including iterative higher-order power methods \citep{kolda2011shifted,kolda2014adaptive}, homotopy continuation \citep{chen2016computing}, semidefinite programming \citep{cui2014all}, and iterative Newton-based methods \citep{jaffe2017newton,guo2017modified}.
We conclude this section with the definition of 
\emph{Newton-stable} eigenpairs \citep{jaffe2017newton}
which are most relevant to our work.

\paragraph{Newton-stable eigenpairs}
Equivalently to \eqref{eq:fixed_point_def}, eigenpairs 
of $\T$ 
can also be characterized by the function
$\gv:\R^d\to\R^d$,
\begin{equation}
\label{eq:g_function}
\gv(\uv) = \T(I,\uv,\uv) - \T(\uv,\uv,\uv)\cdot\uv.
\end{equation}
It is easy to verify that a pair $(\uv,\lambda)$ with $\nrm{\uv}=1$ is an eigenpair of $\T$ if and only if $\gv(\uv)=\bm 0$ and $\lambda=\T(\uv,\uv,\uv)$.
The stability of an eigenpair is determined by its Jacobian matrix $ \nabla \gv(\uv)\in\R^{d\times d}$, more precisely, by its projection into the $d-1$ dimensional subspace orthogonal to $\uv$.
Formally, let $\proj_{\uv}\in \R^{d \times (d-1)}$ be a matrix with $d-1$ orthonormal columns that span the subspace orthogonal to $\uv$
and define the $(d-1)\times (d-1)$ projected Jacobian matrix
\beqn
\label{eq:proj_hessian}
\hess(\uv) = \proj_{\uv}^\top \nabla \gv(\uv) \proj_{\uv}.
\eeqn

\begin{definition}
\label{def:stable_eig}
An eigenpair $(\uv,\lambda)$ of $\T\in\R^{d\times d\times d}$ is {Newton-stable} if the matrix $\hess(\uv)$ has full rank $d-1$.
\end{definition}

The homotopy continuation method in \citet{chen2016computing} is guaranteed to compute all the Newton-stable eigenpairs of a tensor.
Alternatively, 
Newton-stable eigenpairs are attracting fixed points for the iterative orthogonal Newton correction method (O--NCM) in \citet{jaffe2017newton}. Moreover, O--NCM converges to any Newton-stable eigenpair at a quadratic rate given a sufficiently close initial guess.
Finally, for a generic tensor, all its eigenpairs are Newton-stable.

%
%

\section{Learning in the noiseless case}
\label{sec:method}
To motivate our approach for estimating the matrix $W$ it is instructive to first consider the ideal noiseless case where $\sigma=0$.
In this case, model \eqref{eq:model} takes the form $\fv=W^\top \hv$.
Our problem then becomes that of factorizing the observed matrix $\samp = [\fv_1,\dots, \fv_n] \in \R^{m \times n}$ of $n$ samples into a product of real and binary low-rank matrices,
\beqn
\text{Find }\,\,\,
W\in\R^{d\times m},
H\in\{0,1\}^{d\times n}
\,\,\,\,\text{s.t.}\,\,\,\,
\samp= W^\top H.
\label{eq:exact_problem}
\eeqn
To be able to recover $W$ we first need conditions under which the decomposition of $\samp$ into $W$ and $H$ is unique.
Clearly, such a factorization can be unique at most up to a permutation of its components; we henceforth ignore this degeneracy.
A sufficient condition for uniqueness, similar to the one posed in \citet{slawski2013matrix}, is that $H$ is \emph{rigid}. 
Formally, 
$H\in\{0,1\}^{d\times n}$ 
is {rigid} if any non-trivial linear combination of its rows yields a non-binary vector:
$\forall\uv\neq\bm 0$,
\beqn
\bm u^\top H \in\{0,1\}^{n}
\quad\Leftrightarrow\quad
\bm u \in 
\{\bm e_i\}_{i=1}^d 
.
\label{eq:rigid}
\eeqn
Condition \eqref{eq:rigid} is satisfied, for example, when the columns of $H$ include $\bm e_i$ and $\bm e_i + \bm e_j$ for all $i\neq j\in[d]$.

The following proposition, similar in nature to the (affine constrained) uniqueness guarantee in \citet{slawski2013matrix},
shows that under 
condition \eqref{eq:rigid} the factorization in \eqref{eq:exact_problem} is unique and fully characterized by the binary constraints.
\begin{proposition}
\label{lem:pseudo}
Let $\samp=W^\top H$ with $H\in\{0,1\}^{d\times n}$ rigid and $W\in\R^{d\times m}$ full rank with  $m\geq d$.
Let $\Winv\in\R^{m\times d}$ be the unique right pseudo-inverse of $W$ so $W\Winv  = I_d$.
Then $W$ and $H$ are unique and for all 
$\vv \in \spn(\samp)\setminus\{\bm 0\}$,
\beqn
\vv^\top \samp \in\{0,1\}^{n}
\quad\Leftrightarrow\quad
\vv \in \Winv 
.
\label{eq:rigid_F}
\eeqn
\end{proposition}

Hence, under the rigidity condition \eqref{eq:rigid}, the matrix factorization problem in \eqref{eq:exact_problem} is equivalent to the problem of finding the \emph{unique} set $\Vast=\{\vast_1,\dots,\vast_d\}\subseteq \spn(\samp)$ of $d$ non-zero vectors that satisfy the binary constraints ${\vast_i}^\top \samp\in\{0,1\}^n$.
The weight matrix is then $W = (\Winv)^\dagger$.

\paragraph{Algorithm outline} 
We recover $\Vast$ via a two step procedure.
First, a finite set $V = \{\vv_1,\vv_2,\dots\}\subseteq\spn(\samp)$ of \textit{candidate} vectors is computed with a guarantee that $\Vast\subseteq V$.
Specifically, $V$ is computed from the set of eigenpairs of a $d\times d\times d$ tensor, constructed from the low order moments of $\samp$.
Typically, the size of $V$ will be much larger than $d$, so in the second step $V$ is \textit{filtered} by selecting all $\vv \in V$ that satisfy
 ${\vv^\top} \samp \in \{0,1\}^n$.

Before describing the two steps in more detail we first state the additional non-degeneracy conditions we pose.
To this end, denote the unknown first, second, and third order moments of the latent binary vector $\hv$ by
\begin{salign}
\hfo &= \E[\hv] \in \R^d,
\\
\hso &= \E[\hv \otimes \hv]\in\R^{d\times d},
\\
\hto &= \E[\hv \otimes \hv \otimes \hv]\in\R^{d\times d\times d}.
\label{eq:latent_moments_3}
\end{salign}

\paragraph{Non-degeneracy conditions}
We assume the following:
\begin{itemize}
\item[\textit{(I)}] $H$ is rigid. 
\item[\textit{(II)}] 
$\rank(  2\hto(I,I,\bm e_i) - \hso)=d$
for all $i\in[d]$.
\end{itemize}

Condition \textit{(I)} implies 
$\rank(\hso)= \rank(H H^\top)=d$. This in turn implies 
$p_i=\E[h_i] > 0$ for all $i\in[d]$
and
that at most one variable $h_i$ has $p_i=1$. Such an ``always on'' variable can model a fixed bias to $\xv$.
As far as we know, condition \textit{(II)} is new and its nature will become clear shortly.

We now describe each step of our algorithm in more detail.

\paragraph{Computing the candidate set}
To compute a set $V$ that is guaranteed to include the columns of $\Winv$ we make use of the second and third order moments of $\fv$,
\begin{salign}
\fso &= \E[\fv \otimes \fv]\in\R^{m\times m},
\\
\fto &= \E[\fv \otimes \fv \otimes \fv]\in\R^{m\times m\times m}.
\label{eq:moments}
\end{salign}
Given a large number of samples $n\gg 1$, these can be easily and accurately estimated from the sample $\samp$.
For simplicity, in this section we consider the population setting where $n\to \infty$, so $\fso$ and $\fto$ are known exactly. 
$\fso$ and $\fto$ are related to the unknown second and third order moments of 
$\hv$ in \eqref{eq:latent_moments_3} via \citep{anandkumar2014tensor}
\beqn
\label{eq:moments_relations}
\fso = W^\top \hso W,
\qquad
\fto = \hto(W,W,W).
\eeqn
Since both $\hso$ and $W$ are full rank, the number of latent units can be deduced by $\rank(\fso)=d$.
Since $\hso$ is positive definite, there is a whitening matrix $\wm\in\R^{m \times d}$ such that 
\begin{equation}\label{eq:M_def}
\wm^\top \fso\wm = I_d.
\end{equation}
Such a $\wm$ can be computed, for example, by an eigen-decomposition of $M$.
Although $\wm$ is not unique, any $\wm\subseteq\spn(\fso)$ that satisfies \eqref{eq:M_def} suffices for our purpose.
Define the $d\times d\times d$ lower dimensional whitened tensor
\begin{equation}
\wto = \fto(\wm,\wm,\wm).
\label{eq:tensor_wight}
\end{equation}
Denote the set of
 \emph{eigenpairs} of $\wto$ by
\begin{equation}
\label{eq:eig_set}
U = \{(\uv,\lambda)
\in\sphere_{d-1}\times\R_+
: \wto(I,\uv,\uv) = \lambda \uv\}.
\end{equation}
Our set 
of candidates is then
\begin{equation}
\label{eq:subsets}
\EVset = 
\{K\uv/\lambda: (\uv,\lambda) \in U
\text{ with } \lambda \geq 1\}
\subseteq \R^m.
\end{equation}
The following lemma
shows that under 
condition \textit{(I)} 
the set 
$V$ 
is guaranteed to contain 
the $d$ columns of 
$W^\dagger$.
\begin{lemma} 
\label{lem:eig_to_W}
Let $\wto$ be the tensor in \eqref{eq:tensor_wight} corresponding to model \eqref{eq:model} with $\sigma=0$ and let $V$ be as
in \eqref{eq:subsets}.
If condition \textit{(I)} holds then $\Vast\subseteq V$.
In particular, 
each $(\uv_i,\lambda_i)$ in 
the set of $d$ relevant eigenpairs
\beqn
\label{eq:U_ast}
\Uast = \{ (\uv,\lambda)
 \in U
: \wm \uv/\lambda \in \Vast\}
\eeqn
has the eigenvalue $\lambda_i=1/\sqrt{p_i}\geq 1$
 where $p_i=\E[h_i]>0$.
\end{lemma}


\paragraph{Computing the tensor eigenpairs}
By Lemma \ref{lem:eig_to_W}, we may construct a candidate set $V$ that contains $\Vast$ by first calculating the set $U$ of eigenpairs of $\wto$.
Unfortunately, computing the set of {all} eigenpairs of a general symmetric tensor is computationally hard \citep{hillar2013most}.
Moreover, besides the $d$ columns of $\Vast$, the set $\EVset$ in \eqref{eq:subsets} may contain many spurious candidates, as the number of eigenpairs of $\wto$ is typically $O(2^d)$ which is much larger than $d$ \citep{cartwright2013number}.

Nevertheless, as discussed in Section \ref{sec:tensor preliminaries}, several methods have been proposed for computing \emph{some} eigenpairs of a tensor under appropriate stability conditions.
The following lemma highlights the importance of condition \textit{(II)} for the stability of the eigenpairs in $\Uast$.
Note that conditions \textit{(I)-(II)} do not depend on $W$, but only on the distribution of the latent variables $\hv$.


\begin{lemma} 
\label{lem:eig_stability}
Let $\wto$ be the whitened tensor in \eqref{eq:tensor_wight} corresponding to model \eqref{eq:model} with $\sigma=0$.
If conditions \textit{(I)-(II)} hold, then all $(\uv,\lambda) \in \Uast$ are Newton-stable eigenpairs of $\wto$.
\end{lemma}
Hence, under conditions \textit{(I)-(II)}, 
the homotopy method in \citet{chen2016computing}, or alternatively the 
O--NCM with a sufficiently large number of random initializations \citep{jaffe2017newton}, 
are guaranteed to compute a candidate set which includes 
all the columns of $\Winv$. 
%
The next step is to extract $\Vast$ out of $V$.

\paragraph{Filtering} 
As suggested by Eq.\ \eqref{eq:rigid_F} we select the subset of vectors  $\Vsub\subseteq V$ that satisfy the binary constraints,
\begin{equation}
\label{eq:filtering_noiseless}
\Vsub=  \{\vv \in V : {\bm v^T} \samp \in \{0,1\}^n\}.
\end{equation}
Indeed, 
under condition \text{(I)},
Proposition \ref{lem:pseudo} implies that $\Vsub = \Vast$ and the weight matrix is thus $W = \Vsub^\dagger$.

Algorithm \ref{algo:noiseless} summarizes our method for estimating $W$ in the noiseless case and has the following recovery guarantee.

\begin{theorem} 
\label{thm:consistency_noiseless}
Let $\samp$ be a matrix of $n$ samples from model \eqref{eq:model} with $\sigma=0$. 
If 
conditions \textit{(I)-(II)} 
 hold, then Algorithm \ref{algo:noiseless} recovers $W$ exactly.
\end{theorem}

We note that when $\sigma=0$
and conditions \textit{(I)-(II)} hold for the \emph{empirical} latent moments
$\hat \hso$ and $\hat \hto$ (rather than $\hso$ and $\hto$),
Algorithm \ref{algo:noiseless} \emph{exactly} recovers $W$ 
when $\fso$ and $\fto$ are 
replaced by their finite sample estimates.
%
The matrix factorization method SHL in \citet{slawski2013matrix} also exactly recovers $W$ in the case $\sigma=0$.
While its runtime is also exponential in $d$,
practically it may be much faster than our proposed
tensor based approach.
This is because SHL constructs a candidate set of size $2^d$ that can be computed by a suitable linear transformation of the \emph{fixed} set $\{0,1\}^d$, as opposed to our candidate set which is constructed by eigenpairs of a $d\times d\times d$ tensor. 
However, SHL does not take advantage of the large number of samples $n$, since only  $m\times d$ sub-matrices of the $m\times n$ sample matrix $\samp$ are used for constructing its candidate set.
Indeed, in the noisy case where $\sigma>0$, SHL has no consistency guarantees and as demonstrated by the simulation results in Section \ref{sec:experiments} it may fail at high levels of noise.
In the next section we derive a robust version of our method that {consistently} estimates $W$ for any noise level $\sigma\geq0$.

\begin{algorithm}[t]   
	\caption{Recover $W$ when $\sigma=0$}
	\label{algo:noiseless}   
	\hspace*{\algorithmicindent} \textbf{Input:} sample matrix $\samp$
	\begin{algorithmic}[1]
		\STATE estimate second and third order moments $\fso$, $\fto$
		\STATE set $d = \rank(\fso)$
		\STATE compute $K\subseteq\spn(\fso)$ such that $\wm^\top \fso \wm = I_d$
		\STATE compute whitened tensor $\wto = \fto(\wm,\wm,\wm)$
		\STATE compute the set $U$ of eigenpairs of $\wto$
		\STATE compute the candidate set $V$ in \eqref{eq:subsets}
		\STATE filter $\Vsub = \{\vv \in V  : {\vv^\top} \samp \in \{0,1\}^n\}$
		\STATE {\bf return} the pseudo-inverse $W = \Vsub^\dagger$
\end{algorithmic}\end{algorithm}

\section{Learning in the presence of noise}
\label{sec:method_noise}
The method in Section \ref{sec:method} to estimate $W$ is clearly inadequate when $\sigma>0$.
However, we now show that  by making several adjustments, the two steps of computing the candidate set and its filtering can be both made robust to noise, yielding a consistent estimator of $W$ for any $\sigma\geq0$.

\paragraph{Computing the candidate set}
As in the case $\sigma=0$ our goal in the first step is to compute a finite candidate set $\Vsig\subseteq \R^m$ that is guaranteed to contain accurate estimates for the $d$ columns of $\Winv$.
To this end, in addition to the second and third order moments $\fso$ and $\fto$ in \eqref{eq:moments}, we also consider the first order moment $\ffo = \E[\fv]$ and define the following 
noise corrected moments,
\begin{salign}
\label{eq:moments_sigma}
\fso_{\sigma} \,=\,\,& \fso - \sigma^2 I_m,
\\
\fto_{\sigma} \,=\,\,& \fto- \sigma^2 \sum_{i=1}^m \Big ( \bm\mu \otimes \bm e_i \otimes \bm e_i \,+ 
\bm e_i \otimes \bm\mu \otimes  \bm e_i +  
\bm e_i \otimes \bm e_i \otimes \bm\mu \Big).
\end{salign}
By assumption, the noise satisfies $\E[\noise_i^3]=0$.
Thus, similarly to the moment equations in \eqref{eq:moments_relations},
the modified moments in \eqref{eq:moments_sigma} are related to these of $\hv$
by \citep{anandkumar2014tensor}
\beqn
\label{eq:moments_relations_noise}
\fso_\sigma = W^\top \hso W,
\qquad
\fto_\sigma = \hto(W,W,W).
\eeqn
Hence, if $\fso_{\sigma}$ and $\fto_{\sigma}$ were known exactly, a candidate set $\Vsig$ that \emph{contains} $\Vast$ could be obtained exactly as in the noiseless case, but with $\fso$ and $\fto$ replaced with $\fso_\sigma$ and $\fto_\sigma$;
namely, first calculate the whitening matrix $\Ksig$ such that $\Ksig^\top \fso_\sigma\Ksig = I_d$ and then compute the eigenpairs of the population whitened tensor 
\beqn
\label{eq:white_tensor_noisy}
\wto_\sigma = \fto_\sigma(\Ksig,\Ksig,\Ksig).
\eeqn
In practice, $\sigma^2$, $d$, $\ffo$, $\fso$ and $\fto$ are all unknown and need to be estimated from the sample matrix $\samp$.
Assuming $m>d$, the parameters $\sigma^2$ and $d$ can be consistently estimated, for example, by the methods in \citet{kritchman2009non}.
For simplicity, we assume they are known exactly.
Similarly, $\ffo$, $\fso$, $\fto$ are consistently estimated by their empirical means,
$\effo$, $\efso$, and $\efto$.
%
%
So, after computing the plugin estimates $\eKsig$ such that $\eKsig^\top \efso_\sigma \eKsig = I_d$ and $\ewto_\sigma = \efto_\sigma(\eKsig,\eKsig,\eKsig)$,
we compute the set $\eUsig$ of eigenpairs of $\ewto_\sigma$ and for some small $0<\lamths =O(n^{-\frac{1}{2}})$ take our candidate set as
\beqn
\label{eq:subsets_noisy}
\eVsig = 
\{\eKsig\uv/\lambda : (\uv,\lambda) \in \eUsig
\text{ with } 
1\!-\!\lamths \leq \lambda \}.
\eeqn
The following lemma shows that under conditions \textit{(I)-(II)} the above procedure is stable to small perturbations.
Namely, for perturbations of order $\delta\ll 1$ in $\wto_\sigma$ and $\Ksig$, the method computes a candidate set $\eVsig$ that contains a subset of $d$ vectors that are $O(\delta)$ close to the columns of $\Vast$.
Furthermore, these $d$ vectors all correspond to Newton-stable eigenpairs of the perturbed tensor and are $\Omega(1)$ separated from the other candidates in $\eVsig$.

\begin{lemma}
\label{lem:candidate_stability}
Let $\Ksig,\wto_\sigma$ be the population quantities in \eqref{eq:white_tensor_noisy} and let $\eKsig,\ewto_\sigma$ be their perturbed versions, inducing the candidate set $\eVsig$ in \eqref{eq:subsets_noisy}.
If conditions \textit{(I)-(II)} hold, then there are $c,\delta_0,\delta_1>0$ such that for all $0\leq \delta \leq \delta_0$ the following holds:
If the perturbed versions satisfy
\beqn
\label{eq:eig_stab_precond}
\max\{\nrm{\ewto_\sigma - \wto_\sigma}_F, \nrm{\eKsig- \Ksig}_F\} \leq \delta,
\eeqn
then any $\vast\in \Vast$ has a unique $\hvv\in \eVsig$ such that
\beqn
\label{eq:ev_up}
\nrm{\hvv - \vast}\leq c\delta.
\eeqn
Moreover, $\hvv$ corresponds to a Newton-stable eigenpair of $\ewto_\sigma$ with eigenvalue $\lambda\geq 1-c\delta$ and for all $\tvv\in\eVsig\setminus\{\hvv\}$, 
\begin{equation}
\nrm{\tvv-\vast}\geq \delta_1
> 2c\delta.
\label{eq:separated}
\end{equation}
\end{lemma}
The proof
is based on the
implicit function theorem \citep{hubbard2015vector};
small perturbations to a tensor result in small perturbations to its Newton-stable eigenpairs.

Now, by the delta method, the plugin estimates $\eKsig$ and $\ewto_\sigma$
are both $O_P(n^{-\frac{1}{2}})$ close to their population quantities,
\begin{salign}
\label{eq:moments_rate}
\nrm{\eKsig- \Ksig}_F & \,=\,  O_P(n^{-\frac{1}{2}}),
\\
\nrm{\ewto_{\sigma} - \wto_{\sigma}}_F &\,=\, O_P(n^{-\frac{1}{2}}).
\end{salign}
By \eqref{eq:moments_rate}, we have that \eqref{eq:eig_stab_precond} holds with $\delta = O_P(n^{-\frac{1}{2}})$.
Hence, by Lemma \ref{lem:candidate_stability}, the eigenpairs of $\ewto_\sigma$ provide a candidate set $\eVsig$ that contains $d$ vectors that are $O_P(n^{-\frac{1}{2}})$ close to the columns of $\Winv$.
In addition, any irrelevant candidate is $\Omega_P(1)$ far away from $\Vast$.
As we show next, these properties ensure that 
with high probability the $d$ relevant candidates can be identified in $\eVsig$.

\paragraph{Filtering}
Given the candidate set $\eVsig$ computed in the first step, our goal now is to find a set $\Vsubsig\subseteq \eVsig$ of $d$ vectors that accurately estimate the $d$ columns of $\Winv$.
To simplify the theoretical analysis, we assume we are given a fresh sample $\samp$ of size $n$ that is independent of $\eVsig$. This can be achieved by first splitting a sample of size $2n$ into two sets of size $n$, one for each step.

Recall that for a vector $\fv$ from model \eqref{eq:model} and any $\vv\in\R^m$
\beqn
\label{eq:signal_proj}
\vv^\top \fv 
= \vv^\top W^\top \hv +  \sigma\vv^\top\noisev.
\eeqn
Obviously, when $\sigma>0$, the filtering procedure in \eqref{eq:filtering_noiseless} for the noiseless case is inadequate, as typically no $\vast\in \Vast$ will exactly satisfy ${\vast}^\top \samp \in \{0,1\}^n$.
Nevertheless, we expect that for a sufficiently small noise level $\sigma$, any $\vv\in\eVsig$ that is close to some $\vast\in\Vast$ will result in $\vv^\top \samp$ that is close to being binary, while any $\vv$ sufficiently far from $\Vast$ will result in $\vv^\top \samp$ that is far from being binary.
A natural measure for how $\vv^\top \samp$ is ``far from being binary'',
similar to the one
used for filtering in \citet{slawski2013matrix}, 
is simply its deviation from its binary rounding,
\beqn
\label{eq:score_round}
\min_{\bv\in\{0,1\}^n} 
\frac{\nrm{\vv^T \samp - \bv}^2}{n\nrm{\vv}^2}.
\eeqn
Eq. \eqref{eq:score_round} works extremely well for small $\sigma$, but fails for high noise levels. 
Here we instead propose a filtering procedure based on the classical Kolmogorov-Smirnov goodness of fit test \citep{lehmann2006testing}. 
As we show below, this approach gives consistent estimates of $W$ for any $\sigma>0$. 

Before describing the test, we first introduce the probabilistic analogue of the rigidity condition \eqref{eq:rigid}.
For any $\uv\in\R^d$, define its corresponding expected binary rounding,
\beq
\expbr(\uv) = \E_{\hv\sim P_{\hv}}\left[\min_{b\in\{0,1\}} (\uv^\top \hv - b)^2\right].
\eeq
Clearly, $\expbr(\bm 0)=0$ and $\expbr(\bm e_i)=0$ for all $i\in[d]$.
We pose the following \emph{expected rigidity} condition:
for all $\uv \neq \bm 0$,
\beqn
\label{eq:rigid_exp}
\expbr(\uv) = 0
\quad \Leftrightarrow \quad
\uv\in
\{\bm e_i\}_{i=1}^d
.
\eeqn
Analogously to the deterministic rigidity condition in \eqref{eq:rigid}, condition \eqref{eq:rigid_exp} is satisfied, for example, when $P_{\hv}(\bm e_i)> 0$ and $P_{\hv}(\bm e_i + \bm e_j)> 0$ for all $i\neq j\in[d]$.

To introduce our filtering test, recall that under model \eqref{eq:model}, $\noisev\sim\mathcal{N}(\bm 0,I_m)$. 
Hence, for any fixed $\vv$, the random variable $\vv^\top \fv$ in \eqref{eq:signal_proj} is distributed according to the following univariate Gaussian mixture model (GMM),
\begin{equation}
\vv^\top \fv \sim
 \sum_{\hv \in\{0,1\}^d} P_{\hv}[\hv]\cdot \mathcal{N}(
 \vv^\top W^\top \hv,
 \sigma^2\nrm{\vv}^2).
\label{eq:mix_all}
\end{equation}
Denote the cumulative distribution function of $\vv^\top \fv$ by $F_{\vv}$.
For general $\vv$, this mixture may have up to $2^d$ distinct components.
However, for $\vast\in\Vast$, it reduces to a mixture of \emph{two} components with means at $0$ and $1$.
More precisely, for any candidate $\vv$ with corresponding eigenvalue $\lambda(\vv)\geq 1$, define the GMM with {two} components
\beqn
\label{eq:gmm2}
(1-\tfrac{1}{\lambda(\vv)^2})\cdot\mathcal{N}(0, \sigma^2\nrm{\vv}^2) 
 + \tfrac{1}{\lambda(\vv)^2} \cdot\mathcal{N}(1, \sigma^2 \nrm{\vv}^2).
\eeqn
Denote its cumulative distribution function by $G_{\vv}$.
The following lemma
shows that under 
condition \eqref{eq:rigid_exp}, $G_{\vv}$ fully characterizes the columns of $\Winv$.
\begin{lemma}
\label{lem:test_vast}
Let $\Ksig,\wto_\sigma$ be the population quantities in \eqref{eq:white_tensor_noisy} and let $\Vsig$ be the set of population candidates as computed from the eigenpairs of $\wto_\sigma$.
If conditions \textit{(I)-(II)} and the expected rigidity condition \eqref{eq:rigid_exp} hold, then for any $\vv\in \Vsig$ and its corresponding eigenvalue $\lambda(\vv)$,
\beq
F_{\vv} = G_{\vv}
\quad\Leftrightarrow\quad
\vv\in\Vast.
\eeq
\end{lemma}


Given the empirical candidate set $\eVsig$, Lemma \ref{lem:test_vast} suggests ranking 
all $\hvv\in\eVsig$ according to their goodness of fit to $G_{\hvv}$ and taking the $d$ candidates with the best fit.
More precisely, given a sample $\samp=[\fv_1,\dots,\fv_n]$ that is independent of $\eVsig$, for each candidate $\hvv\in\eVsig$ we compute the empirical cumulative distribution function,
\begin{equation*}
\eCom_{\hvv}(t) = \frac{1}{n} \sum_{j=1}^n \chr\{\hvv^\top\fv_j\leq t\},
\qquad t\in\R,
\end{equation*}
and calculate its Kolmogorov-Smirnov score
\beqn
\label{eq:emo_score}
\score_n(\hvv) = \sup_{t\in\R} | \eCom_{\hvv}(t) - G_{\hvv}(t)|.
\eeqn
Our estimator $\Vsubsig\subseteq\eVsig$ for $\Winv$ is then the set of $d$ vectors with the smallest scores $\score_n(\hvv)$.
The estimator for $W$ is the pseudo-inverse, $\Wemp=\Vsubsig^\dagger$.

The following lemma shows that for sufficiently large $n$, $\score_n(\hvv)$ accurately distinguishes between $\hvv\in\eVsig$ that are close to the columns of $\Winv$ from these that are not.

\begin{lemma}
\label{lem:cont_map_thm}
Let $\vast\in\Vast$ and $\hvv_{(1)},\hvv_{(2)},\dots$ a sequence of random vectors such that $\nrm{\hvv_{(n)} - \vast} = O_P(n^{-\frac{1}{2}}).$
Then, 
$$\score_n(\hvv_{(n)}) = o_P(1).$$
In contrast, if $\min_{\vast\in\Vast}\nrm{\hvv_{(n)} - \vast} = \Omega_P(1)$,
then 
$$\score_n(\hvv_{(n)}) = \Omega_P(1),$$
provided the expected rigidity condition \eqref{eq:rigid_exp} holds.
\end{lemma}

Lemma \ref{lem:cont_map_thm}
 follows from 
classical and well studied properties 
of the Kolmogorov-Smirnov test, see for example 
\citet{lehmann2006testing, billingsley2013convergence}.

Algorithm \ref{algo:noisy} summarizes our method for estimating $W$ in the general case where $\sigma\geq0$ and $n<\infty$.
The following theorem establishes its consistency.
\begin{algorithm}[t]   
	\caption{Estimate $W$ when $\sigma>0$ and $n<\infty$}
	\label{algo:noisy}
	\hspace*{\algorithmicindent} \textbf{Input:} sample matrix $\samp\in\R^{m\times n}$ and $0<\lamths\ll 1$
	\begin{algorithmic}[1]
	\STATE estimate number of hidden units $d$ and noise level $\sigma^2$
	\STATE compute empirical moments 
	$\effo$, $\efso$ and $\efto$ and plugin moments $\efso_\sigma$ and $\efto_\sigma$ of \eqref{eq:moments_sigma}
	\STATE compute $\eKsig$ such that $\eKsig^\top \efso_\sigma \eKsig = I_d$	
	\STATE construct $\ewto_\sigma = \efto_{\sigma}(\eKsig,\eKsig,\eKsig)$ 
	\STATE compute the set $\eUsig$ of eigenpairs of $\ewto_\sigma$
	\STATE compute the candidate set $\eVsig$ in \eqref{eq:subsets_noisy}
	\STATE for each $\hvv\in\eVsig$ compute its KS score $\score_n(\hvv)$ in \eqref{eq:emo_score}
	\STATE select $\Vsubsig\subseteq \eVsig$ of $d$ vectors with smallest $\score_n(\hvv)$
	\STATE \textbf{return} the pseudo-inverse $\Wemp = \Vsubsig^\dagger$	
\end{algorithmic}\end{algorithm}

\begin{theorem}
Let $\fv_1,\dots,\fv_n$ be $n$ i.i.d.\ samples from model \eqref{eq:model}.
If conditions \textit{(I)-(II)} and the expected rigidity condition \eqref{eq:rigid_exp} hold, then the estimator $\Wemp$ computed by Algorithm \ref{algo:noisy} is consistent, 
achieving the parametric rate,
\beq
\Wemp = W + O_P(n^{-\frac{1}{2}}).
\eeq
\end{theorem}

\paragraph{Runtime}
The runtime of Algorithm \ref{algo:noisy} is composed of three main parts.
First, $O(n m^3)$ operations are needed to compute all the relevant moments from the data and to construct the $d\times d\times d$ whitened tensor $\ewto_\sigma$.
The most time consuming task is computing the eigenpairs of $\ewto_\sigma$, which can be done by either the homotopy method or O--NCM.
Currently, no runtime guarantees are available for either of these methods. In practice, since there are $O(2^d)$ eigenpairs, these methods spend $O(2^d\cdot poly(d))$ operations in total.
Finally, since there are $O(2^d)$ candidates and each KS test takes $O(d n)$ operations \citep{gonzalez1977efficient}, the filtering procedure runtime is $O(d 2^d n)$.


\paragraph{Power-stability and orthogonal decomposition}

The exponential runtime of our algorithm stems from the fact that the
set $U_N$ of Newton-stable eigenpairs of $\wto_\sigma$ is 
typically exponentially large.
Indeed, the 
above algorithm
becomes intractable for large values of $d$.
However,
in some 
cases,
the set $\Uast$ of $d$ relevant eigenpairs 
has additional structure so that a smaller candidate set 
may be computed instead of $U_N$.
%
Specifically, consider the subset 
$U_P\subseteq U_N$
of \emph{power-stable} eigenpairs
of $\wto_\sigma$.
%
\begin{definition}
\label{def:power_stable_eig}
An eigenpair $(\uv,\lambda)$ 
is {power-stable} if its projected Jacobian $\hess(\uv)$ is either positive or negative definite.
\end{definition}

Typically, 
the number of power-stable eigenpairs is significantly smaller than the number of Newton-stable eigenpairs.%
\footnote{We currently do not know whether the number of power-stable eigenpairs of a generic tensor is polynomial or exponential in $d$.}
In addition, $U_P$ can be computed by 
the shifted higher-order power method 
\citep{kolda2011shifted,kolda2014adaptive}.

%
%
Similarly to Lemma \ref{lem:eig_stability}, one can show that 
$U_P$ is guaranteed to contain $\Uast$
whenever the following stronger version of 
condition \textit{(II)} holds:
for all $(\uv_i,\lambda_i)\in\Uast$, the matrix
\beqn
\label{eq:power_condition}
(W\wm\proj_{\uv_i})^\top
(2\hto(I,I,\bm e_i) - \hso)  (W\wm\proj_{\uv_i})
\eeqn
is either positive-definite or negative-definite.



As an  example, consider the case where 
$P_{\hv}$ has the support 
$\hv\in I_d$.
Then model \eqref{eq:model} corresponds to a GMM with $d$ spherical components with linearly independent means. 
%
%
In this case, both $\hso$ and $\hto$ are diagonal with 
$\hfo$
on their diagonal.
Thus, the matrices in \eqref{eq:power_condition} take the form
$
-\proj_{\bm e_i}^\top \diag(\hfo)\proj_{\bm e_i},
$
which by condition \textit{(I)} are all negative-definite.
%
%
In fact, in this case, $\wto_\sigma$ has an orthogonal decomposition and the $d$ orthogonal eigenpairs in $\Uast$ are the \emph{only} 
negative-definite power-stable eigenpairs of $\wto_\sigma$ \citep{anandkumar2014tensor}.
Similarly, when $P_{\hv}$ is a product distribution, the same orthogonal structure appears if the \emph{centered} moments of $\fv$ are used instead of $\fso$ and $\fto$.
As shown in \citet{anandkumar2014tensor}, the power method, accompanied with a deflation procedure, decompose an orthogonal tensor in polynomial time, thus implying
an efficient algorithm in these cases.

\section{Experiments}
\label{sec:experiments}
We demonstrate our method in two scenarios: (I) simulations from the exact binary model  \eqref{eq:model};
and (II) learning a common population genetic admixture model. 
Code to reproduce the simulation results can be found at \url{https://github.com/arJaffe/BinaryLatentVariables}.



\begin{figure}[t]
	\centering
	\includegraphics[width=0.45\linewidth]{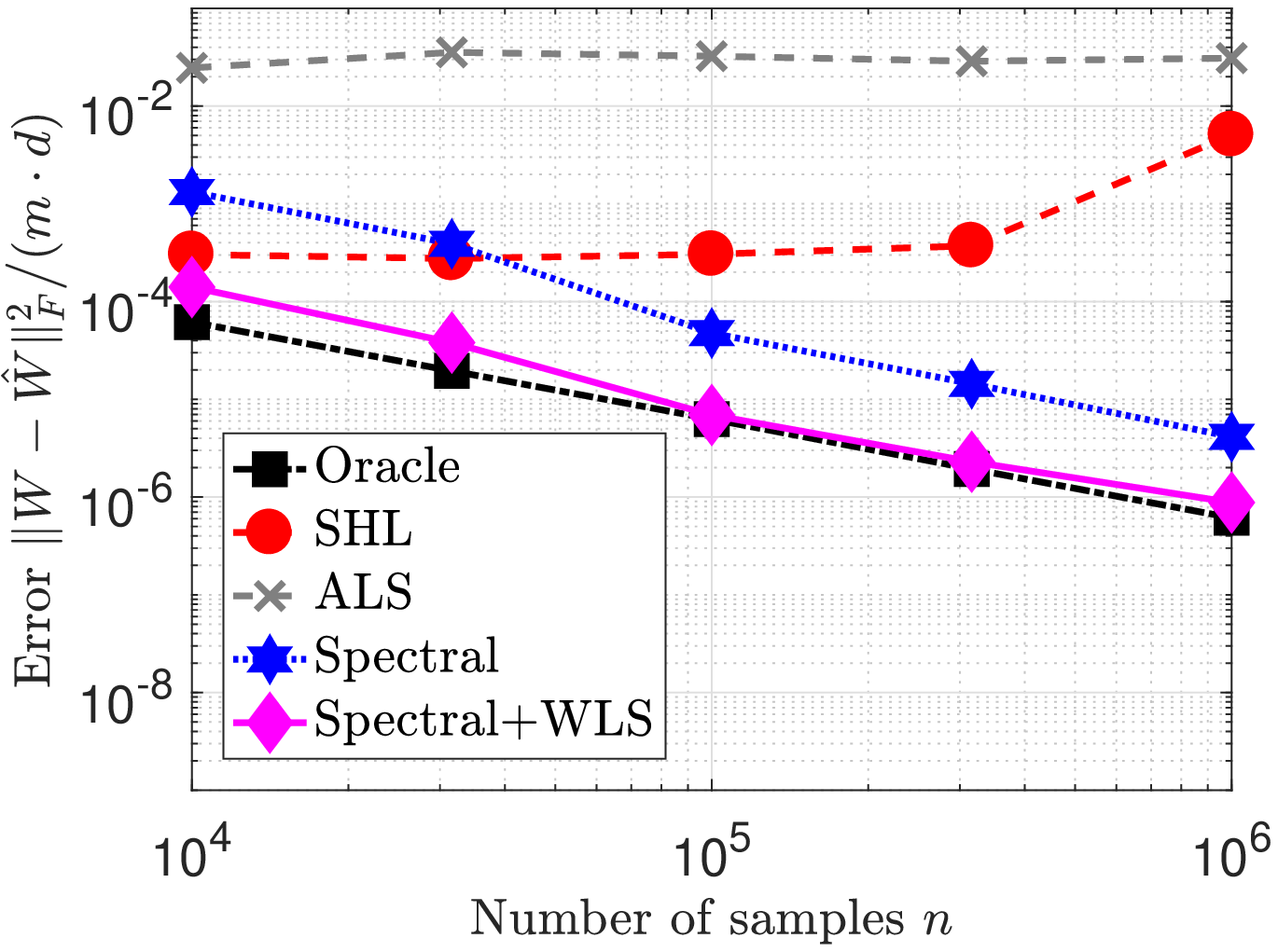}
	\includegraphics[width=0.45\linewidth]{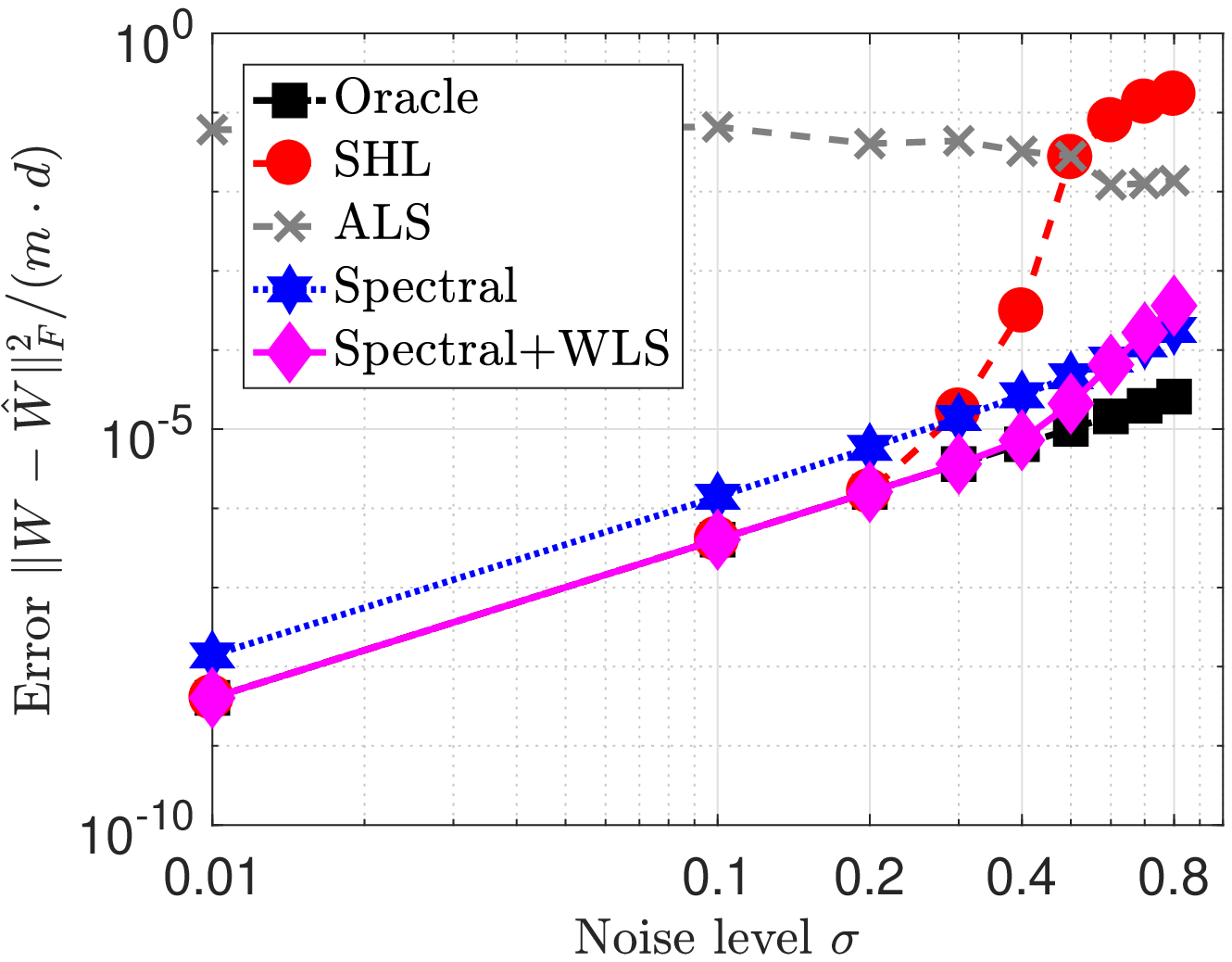}	
	\caption{Left panel: Error in $W$ vs.\ sample size $n$ with $\sigma=0.4$.
	Right panel: Error in $W$ vs.\ noise level $\sigma$ with $n=10^5$.
	}
	\label{fig:guassin_n_scan}
\end{figure}

\paragraph{Simulations} 
We generated $n$ samples from model \eqref{eq:model} with $d=6$ 
hidden units,
$m=30$ observable features, and Gaussian noise $\noisev\sim\mathcal N(\bm 0, I_m)$.
The $m$ columns of $W$ were drawn uniformly from the unit sphere $\sphere_{d-1}$.
Fixing a mean vector $\bm a\in\R^d$ and a covariance matrix $R\in\mathbb \R^{d\times d}$, 
each hidden vector $\hv$ was generated independently by first drawing $\bm r \sim \mathcal{N}(\bm a, R)$ and then taking its binary rounding.

Figure \ref{fig:guassin_n_scan} shows the error, in Frobenius norm, 
averaged over $50$ independent realizations of $X$ as a function of $n$ (upper panel) and $\sigma$ (lower panel) for five methods: (i) our spectral approach, detailed in Algorithm \ref{algo:noisy} (Spectral); (ii) Algorithm \ref{algo:noisy} followed by an additional single weighted least square step
detailed in Appendix \ref{sec:weighted_ls}
(Spectral+WLS); (iii) SHL, the matrix decomposition approach of \citet{slawski2013matrix}\footnote{Code taken from  \url{https://sites.google.com/site/slawskimartin/code}. For each realization, we aggregated over $50$ runs of SHL and chose the output $H,W$ that minimized $\|X - W^\top H\|_F$.};  (iv) ALS with a random starting point (see Appendix \ref{sec:als_method}); and (v) an oracle estimator that is given the exact matrix $H$ and computes $W$ via least squares. 

As one can see, as opposed to SHL, our method is consistent for $\sigma>0$ and achieves an error rate $O(n^{-\frac{1}{2}})$
corresponding to a slope of $-1$ in the left panel of Fig. \ref{fig:guassin_n_scan}.
In addition, as seen in the right panel of Fig. \ref{fig:guassin_n_scan}, at low level of noise our method is comparable to SHL, whereas at high level of noise it is far more accurate.
Finally, adding a weighted least square step reduces the error for low noise levels, but increases the error for high noise levels. 


               

\begin{figure}[t]
\hspace{3cm}
	\includegraphics[width=0.5\linewidth]{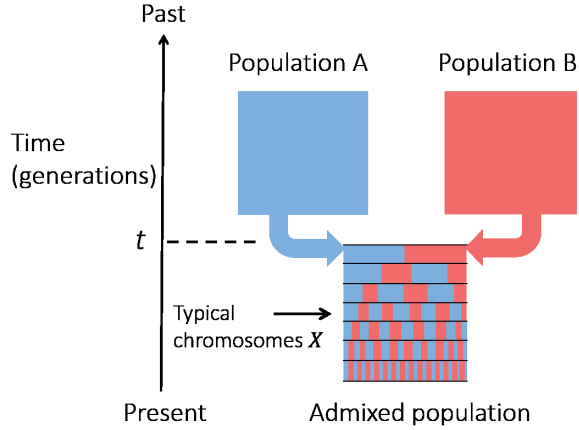}
	\caption{Illustration of the admixture model. 
	}
	\label{fig:admixture_model}
\end{figure}
\paragraph{Population genetic admixture} 	
We present an application of our method to a fundamental problem in population genetics, known as admixture, illustrated in Fig. \ref{fig:admixture_model}. 
Admixture
refers to the mixing of $d\geq 2$ ancestral populations that were long separated, e.g., due to geographical or cultural barriers  \citep{pritchard2000inference,alexander2009fast,li2008worldwide}.
The observed data $X$ is an $m \times n$ matrix where $m$ is the number of modern ``admixed'' individuals and $n$ is the number of relevant locations in their DNA,  known as SNPs.
Each SNP corresponds to two alleles and
different individuals may have different alleles.
Fixing a reference allele for each location, 
$X_{ij}$
takes values in $\{0,\frac{1}{2},1\}$ 
according to 
the number of reference alleles appearing in the genotype of individual $i\in[m]$ at locus $j\in[n]$.

Given the genotypes $X$, an important problem in population genetics
is to estimate the following two quantities.
The allele frequency matrix $H \in [0,1]^{d \times n}$ whose 
entry $H_{kj}$ is the frequency of the reference allele
at locus $j\in[n]$ in ancestral population $k\in[d]$;
and the admixture proportion matrix $W \in [0,1]^{d \times m}$ whose columns sum to $1$ and its 
entry $W_{ki}$
is the proportion of individual $i$'s genome that was inherited from population $k$.

A common model for $X$ in terms of $W$ and $H$ is to assume that the number of alleles $2X_{ij}\in\{0,1,2\}$ is the sum of two i.i.d.\ Bernoulli random variables with success probability 
$F_{ij} = \sum_{k = 1}^d W_{ki}H_{kj}$.
Namely,
\begin{equation*}
X_{ij} | H \sim \tfrac{1}{2}\cdot\text{Binomial}(2,F_{ij}).
\end{equation*}
Note that under this model 
\beqn
\label{eq:bernoulli}
\E[X|H] = F =W^\top H.
\eeqn
Although \eqref{eq:bernoulli} has similar form to model \eqref{eq:model}, there are two main differences;
the noise is not normally distributed
and
the matrix $H$ is non-binary.
Yet, 
model \eqref{eq:model} is expected to be
 a good approximation whenever various alleles are rare in some populations but  abundant
in others. Specifically, for ancestral populations that have been long separated, some alleles may become \textit{fixed} in one population (i.e., reach frequency of 1) while being totally absent in others.

\paragraph{Genetic simulations} 
We followed a standard simulation scheme appearing, for example, in \citet{xue2017time,gravel2012population,price2009sensitive}.
First, using SCRM \cite{scrm}, 
we simulated $d=3$ ancestral populations separated for $4000$ generations and generated the genomes of $40$ individuals for each. 
$H$ was then computed as the frequency of the reference alleles in each population.
Next, the columns of
 $W$ were sampled from a symmetric Dirichlet distribution with parameter $\alpha \geq 0$. 
Finally, the genomes of $m=50$ admixed individuals were generated as mosaics of genomic segments of individuals from the ancestral populations with proportions $W$. 
The mosaic nature of the admixed genomes is an important realistic detail, due to the \textit{linkage} (correlation) between SNPs \cite{xue2017time}. 
A detailed description is in Appendix \ref{sec:genetic_admixture}.

\begin{figure}[t]
	\centering
	\includegraphics[width=0.45\linewidth]{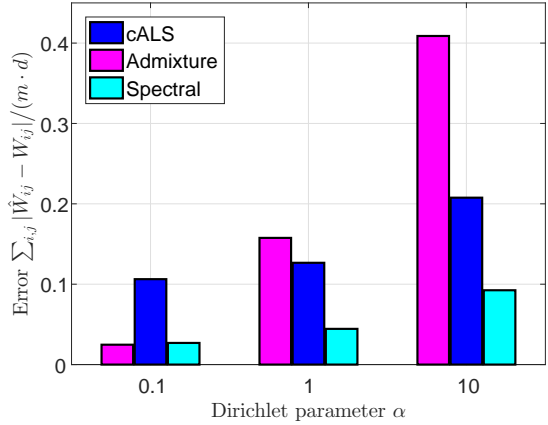}
	\caption{Average absolute error in $\hat W$ vs. Dirichlet parameter $\alpha$ for $d=3$ ancestral populations and $m=50$ admixed individuals.}
	\label{fig:results_admixed}
\end{figure}

We compare our algorithm to two methods. The first is Admixture \cite{alexander2009fast}, one of the most widely used algorithms in population genetics, which aims to
maximize the likelihood of $X$. A recent spectral approach is ALStructure \cite{cabreros2017nonparametric}, where 
an estimation of $\spn(W^\top)$ via \citet{chen2015consistent} is followed by constrained ALS iterations of $W$ and $H$. For our method, 
two modification are needed for Algorithm \ref{algo:noisy}.
First, since the distribution of $X_{ij}-\bm w_i^T \bm h_j$ is not Gaussian, the corrected moments $\hat M_\sigma,\hat \M_\sigma$ as calculated by \eqref{eq:moments_sigma} do not satisfy \eqref{eq:moments_relations_noise}. Instead, we implemented a matrix completion algorithm derived in \cite{jain2014learning} for a similar setup, see Appendix \ref{sec:improved_denoising} for more details. In addition,
the filtering process described in Section \ref{sec:method_noise} is no longer valid. However, as $d$ is relatively small, we were able to perform exhaustive search over all candidate subsets of size $d$ 
and choose the one that maximized the likelihood. 

Figure \ref{fig:results_admixed} compares the results of the $3$ methods for $\alpha=0.1,1,10$. The spectral method outperforms Admixture and ALStructure for $\alpha=1,10$ and performs similarly to Admixture for $\alpha=0.1$.

\paragraph{Acknowledgments} 
This research was funded in part by NIH Grant 1R01HG008383-01A1.

\bibliography{bib_linear_moment_model}
\bibliographystyle{plainnat}

\newpage

\appendix

\section{Proof of Proposition \ref{lem:pseudo}}
\label{sec:proof_rigid}
Uniqueness of the factorization readily follows from \eqref{eq:rigid_F} so we proceed 
to prove \eqref{eq:rigid_F}.
First note that $\spn(\samp)= \spn(W^\top) = \spn(\Winv)$.
Since $W$ is full rank, we have $W \Winv=I_d$.
Hence,
\[
(\Winv)^\top \samp = (W \Winv)^\top H = H \in \{0,1\}^{d\times n}.
\]
So any $\vast\in\Vast$ satisfies the binary constraint ${\vast}^\top \samp \in \{0,1\}^n$.
For the other direction, let $\vv\in\spn(\samp)\setminus\{\bm 0\}$ be such that $\vv^\top \samp \in \{0,1\}^n$.
Since $\vv^\top \samp = (W \vv)^\top H$, the rigidity condition \eqref{eq:rigid} implies $W \vv \in \{\bm e_i\}_{i=1}^d$.
Since $W$ is full rank and $\vv\in \spn(\Winv)$, $\vv$ must be a column of $\Winv$.

\section{Proof of Lemma \ref{lem:eig_to_W}}
\label{sec:proof_V_contains_Vast}
Since the vector $\hv$ is binary, its second and third order moments are related as follows.
For all $i,j \in [d]$,
\beqn
\label{eq:bin_constraints} 
\hto_{iij} = 
\hto_{iji} = \hto_{jii} = \E[h_i^2 h_j] = \E[h_i h_j] =
\hso_{ij}.
\eeqn
Since $W$ is full rank, $W \Winv = I_d$.
Hence, applying $\Winv$ multi-linearly on the moment equations in \eqref{eq:moments_relations} we obtain
\beq
\hso &=& (\Winv)^\top \fso {\Winv},
\\
\hto &=& \fto({\Winv},{\Winv},{\Winv}).
\eeq
Thus, the equality in \eqref{eq:bin_constraints}
is equivalent to
\begin{equation}\label{eq:const_2_3}
[\fto({\Winv},{\Winv},{\Winv})]_{iij}
=[(\Winv)^\top \fso {\Winv}]_{ij}.
\end{equation}
Let  $\wWinv\in\R^{d\times d}$ be the full rank matrix that satisfies $\Winv = \wm{\wWinv}$ where $\wm$ is the whitening matrix in \eqref{eq:M_def}.
Then,
\begin{equation}
\fto({\Winv},{\Winv},{\Winv}) = \fto(K \wWinv,K \wWinv,K \wWinv)
= \wto(\wWinv,\wWinv,\wWinv)
\end{equation}
where $\wto$ is the whitened tensor in \eqref{eq:tensor_wight}.
Similarly, by \eqref{eq:M_def},
\beq
(\Winv)^\top \fso {\Winv} = (\wWinv)^\top (K^T \fso K) (\wWinv) = (\wWinv)^\top \wWinv.
\eeq
Inserting these into \eqref{eq:const_2_3}, the matrix $\wWinv$ must satisfy
\beqn
\label{eq:T_tilde_X}
[\wto({\wWinv},{\wWinv},{\wWinv})]_{iij} = [(\wWinv)^\top (\wWinv)]_{ij},\,\,\, \forall i,j\in[d].
\eeqn
The following lemma, proved in Appendix \ref{sec:proof_eig_equations}, shows that Eq.\ \eqref{eq:T_tilde_X} is nothing but a tensor eigen-problem.
Specifically, the columns of $\wWinv$, up to scaling, 
are eigenvectors of $\wto$.

\begin{lemma}\label{lem:part_1}
Let $\wto\in\R^{d\times d\times d}$ be an arbitrary symmetric tensor.
Then, a matrix $Y=[\yv_1,\ldots,\yv_d] \in \R^{d \times d}$ of rank $d$ satisfies 
\eqref{eq:T_tilde_X} if and only if for all $k\in[d]$, $\yv_k = {\uv_k}/{\lambda_k}$, where $(\uv_k,\lambda_k)_{k=1}^d$ are $d$ eigenpairs of $\wto$  with linearly independent $\{\uv_k\}_{k=1}^d$.
\end{lemma}

By Lemma \ref{lem:part_1}, the set of scaled eigenpairs $\{\yv =\uv/\lambda\}$ of $\wto$ is guaranteed to contain the $d$ columns of $\wWinv$.
Since $\Winv=K\wWinv$, the set $\{K \yv\}$ is guaranteed to contain $\Vast$.

To show that each $\yv = \uv/\lambda \in \wWinv$ has $\lambda\geq 1$,
note that 
the vector $K\yv$ is a column of $\Winv$, so $W K\yv = \bm e_i$ for some $i\in[d]$. Hence,
by the definition of the whitened tensor \eqref{eq:tensor_wight} and the moment equation \eqref{eq:moments_relations},
\beq
\wto(\yv, \yv, \yv)  
&=& \fto(K \yv, K \yv, K \yv)
= \hto(W K \yv, W K \yv, W K \yv)
\\
&=&\hto(\bm e_i, \bm e_i, \bm e_i) 
= \hto_{iii} = \E[h_i] \leq 1.
\eeq
On the other hand, since $(\uv,\lambda)$ is an eigenpair of $\wto$ with eigenvalue $\lambda = \wto(\uv, \uv, \uv)$,
\beq
\wto(\yv, \yv, \yv) = \frac{1}{\lambda^3}\wto(\uv, \uv, \uv) = 
\frac{1}{\lambda^2}.
\eeq
By convention, $\lambda\geq 0$. Hence, 
\beqn
\label{eq:lam_p}
\lambda = 1/\sqrt{\E[h_i]}\geq 1,
\eeqn
concluding the proof.

\section{Proof of Lemma \ref{lem:eig_stability}}
\label{sec:proof_eig_stability}

%

Let $(\uv,\lambda) \in U^*$ 
be an eigenpair of 
$\wto$
  such that
$\vast = \wm \uv/\lambda \in \Vast$.
To show Newton-stability we need to show that under conditions \textit{(I)-(II)} the
projected Jacobian matrix
$J_p(\uv) = \proj_{\uv}^\top \nabla \gv(\uv) \proj_{\uv}$
in \eqref{eq:proj_hessian} is full rank $d-1$.

The Jacobian matrix $\nabla\gv(\uv)$ is
\beqn
\nonumber
\nabla\gv(\uv)
&=& 2 \wto(I,I,\uv) - 3\uv \wto(I,\uv,\uv)^\top 
 - \wto(\uv,\uv,\uv) I_d
\\
\label{eq:jacobian_W}
&=&
2 \wto(I,I,\uv) - 3 \lambda \uv \uv^\top - \lambda I_d.
\eeqn
Since $\proj_{\uv}^\top \uv =\bm 0$, the second term in \eqref{eq:jacobian_W} does not contribute to $J_p(\uv)$.
For the first term in \eqref{eq:jacobian_W}, by \eqref{eq:tensor_wight} and \eqref{eq:moments_relations},
\beq
\wto(I,I,\uv) = \fto(K,K,K\uv) = \hto(WK,WK,WK\uv).
\eeq
Since $\vast = K\uv/\lambda$ 
is a column of $\Winv$,
$W \wm \uv = \lambda\bm e_i$ for some $i\in[d]$.
Thus,
\begin{equation}
\wto(I,I,\uv) = \lambda\hto(WK,WK,\bm e_i) 
= \lambda \wm^\top W^\top \hto(I,I,\bm e_i) W \wm.
\end{equation}
For the third term in \eqref{eq:jacobian_W}, by the definition of $\wm$ in
\eqref{eq:M_def},
\beq
I_d = \wm^\top\fso \wm = \wm^\top W^\top\hso W\wm.
\eeq
Putting the last two equalities in \eqref{eq:jacobian_W} and applying the projection $\proj_{\uv}$ we obtain
\beq
J_p(\uv) = 
\proj_{\uv}^\top \nabla \gv(\uv) \proj_{\uv}
= \lambda \proj_{\uv}^\top\wm^\top W^\top (2\hto(I,I,\bm e_i) - \hso) W \wm\proj_{\uv}.
\eeq
Since $\lambda\geq 1$ and $W$ and $\wm$ are full rank,
condition \textit{(II)} implies that $J_p(\uv)$ is full rank as well.
Thus, $(\uv,\lambda)$ is a Newton-stable eigenpair of $\wto$.

\section{Proof of Lemma \ref{lem:candidate_stability}}
\label{sec:implicit}
Lemma \ref{lem:candidate_stability} follows from the following lemma
which establishes
the stability of Newton-stable eigenpairs of a tensor $\wto$
to small perturbations $\twto= \wto +\Delta\wto$.
 
\begin{lemma} 
\label{lem:pert_bounds}
Let $(\uv,\lambda)$ be a Newton-stable eigenpair of $\wto$
with $\lambda \geq 1$.
There are $c_1,c_2,\eps_0 >0$ such that
for all sufficiently small $\eps>0$ the following holds.
For any $\twto$ such that
$\nrm{\twto - \wto}_F \leq \eps$
there exists a unique eigenpair $(\tuv,\tlambda)$ of $\twto$ such that
\beqn
\label{eq:pert_bounds}
\nrm{\uv - \tilde \uv} \leq c_1\eps
\quad\text{and}\quad
|\tilde \lambda - \lambda| \leq c_2\eps.
\eeqn
In addition, $(\tuv,\tlambda)$ is Newton-stable and any other eigenvector $\tvv$ of $\twto$ satisfies 
$
\nrm{\tvv-\uv}\geq \eps_0.
$
\end{lemma}

\begin{proof}[Proof of Lemma \ref{lem:pert_bounds}.]
For a tensor $\T\in\R^{d\times d\times d}$ let $\tv\in \R^{s}$ be the vector of $s=d^3$ entries $\{\T_{ijk}\}$.
Define the function $\bm Q:\R^{d+s}\to\R^d$ by
\beq
\bm Q\colvec{\vv}{\tv}  
= \T(I,\vv,\vv) - \T(\vv,\vv,\vv)\cdot\vv.
\eeq
%
Note that for any $\tv\in\R^{s}$ and $(\vv,\beta)\in\R^d\times\R$ with $\vv\neq\bm 0$ and $\beta\neq 0$, we have that
$\bm Q\colvec{\vv}{\tv} = \bm 0$
if and only if $(\vv,\beta)$ is an eigenpair of $\tv$ with eigenvalue $\beta=\T({\vv},{\vv},{\vv})$.%
\footnote{This does not precisely hold when $\beta=0$ since $\bm Q\colvec{\vv}{\tv} = \bm 0$ does not imply $\nrm{\bm v}=1$ in this case, but only that $\vv$ is proportional to an eigenvector.}
Denote the gradients of $\bm Q$ with respect to $\vv$ and $\tv$ by
\beq
A(\vv,\tv) &=& \nabla_{\vv} \bm Q\colvec{\vv}{\tv}\in\R^{d\times d},
\\
B(\vv,\tv) &=& \nabla_{\tv} \bm Q\colvec{\vv}{\tv}\in\R^{d\times s}.
\eeq
Let $\wv\in\R^s$ be the vectorization of $\wto$ and let $(\uv,\lambda)\in\sphere_{d-1}\times\R_+$ be a Newton-stable eigenpair of $\wv$ with $\lambda \geq 1$.
Since $\uv$ is Newton-stable and $\lambda > 0$, $ A(\uv,\wv)$ is invertible.
In addition, the following $(d+s) \times (d+s)$ matrix
is invertible,
\beq
D(\uv,\wv) =
\left(
\begin{matrix}
A(\uv,\wv)
& B(\uv,\wv)
\\
0 & I_s
\end{matrix}
\right).
\eeq
Let $\g_D = 1/\nrm{D(\uv,\wv)^{-1}}>0$ be the smallest singular value of $D(\uv,\wv)$ 
and let $L_D<\infty$ be the Lipschitz constant of 
$\nabla \bm Q(\vv,\tv)= [ A(\vv,\tv),  B(\vv,\tv)] \in \R^{d\times (d+s)}$
in a small neighborhood 
of $(\uv,\wv)$,
namely, $\forall (\vv,\tv),(\tvv, \ttv)$ in the neighborhood,
	\beq
	\nrm{\nabla\bm Q(\vv,\tv) - \nabla\bm Q(\tvv, \ttv)} \leq 
	L_D\nrm{({\vv},{\tv}) - (\tvv, \ttv)}.
\eeq
Let $B_\eps(\wv)\subset\R^s$ be the ball of radius $\eps$ centered at $\wv$.
Then by the implicit function theorem \cite{hubbard2015vector}, 
for any $\eps\leq \eps_1 := \g_D^2/(2L_{D} )$, 
there exists a \emph{unique} continuously differentiable mapping
$\tuv : B_\eps(\wv) \to B_{2\eps/\g_D}(\uv)$
such that
$Q\colvec{\tuv(\twv)}{\twv}=\bm 0$ for all $\twv\in B_\eps(\wv)$.
In other words, for any $\twv$ such that $\nrm{\twv-\wv}\leq \eps$, there exist a unique vector $\tuv$ in all $B_{2\eps/\g_D}(\uv)$ that is an eigenpair of $\twv$.
Equivalently, for $\twto$ such that $\nrm{\twto-\wto}_F\leq \eps$,
there exists a unique eigenvector $\tuv$ of $\twto$ 
such that 
\beqn
\label{eq:eigvec_deviation}
\nrm{\tuv-\uv} \leq {2}\eps/{\g_D} := c_1 \eps.
\eeqn
The bound on $|\tlambda-\lambda|$ readily follows from \eqref{eq:eigvec_deviation}.
Indeed, let $q:\R^{d+s}\to\R$ be $q\colvec{{\bm v}}{{\bm t}} = \T(\bm v,\bm v,\bm v)$ and let $L_\lambda$ be the Lipschitz constant of $q$ in the neighborhood of $(\uv,\wv)$.
Then, 
\beq
|\tlambda - \lambda|
&=& |q\colvec{\tuv}{\twv} - q\colvec{\uv}{\wv}|
\\
&\leq& L_\lambda \sqrt{\nrm{\tuv - \uv}^2+\nrm{\twv-\wv}^2}
\\
&\leq& L_\lambda\sqrt{\frac{2}{\g_D} + 1}\cdot\eps := c_2\eps.
\eeq

As for the Newton-stability of $\tuv$,
let $r:\R^{d+s}\to\R_+$ be
$r(\vv,\tv)=1/\nrm{A(\vv,\tv)^{-1}}$,
the minimal singular value of $A(\vv,\tv)$.
Since $(\uv,\lambda)$ is a Newton-stable eigenpair of $\wv$,
$\exists\g_A>0$ such that $r(\uv,\wv)\geq \g_A$.
Let $L_\g$ be the Lipschitz constant of $r(\vv,\tv)$ in the neighborhood \citep{golub2012matrix}.
Then, for $\eps \leq \eps_2 := \g/(2L_\g)$,
we have $r(\tuv,\twv)\geq \g_A/2>0$, so $(\tuv,\tlambda)$ is a Newton-stable eigenpair of $\twv$.

Finally, we show that any other eigenvector $\tvv$ of $\twto$ is
apart from $\uv$.
Since $\tuv$ is Newton-stable, there exists $\eps_0>0$ such that
$\nrm{\tvv-\tuv}\geq 2\eps_0$ for any other eigenvector $\tvv$.
Hence, for $\eps \leq \eps_0$, 
\beq
\nrm{\tvv-\uv}\geq \big|\nrm{\tvv-\tuv} - \nrm{\tuv-\uv}\big| \geq \nrm{\tvv-\tuv} - \eps \geq \eps_0.
\eeq
Taking $\eps \leq \min\{\eps_0,\eps_1,\eps_2\}$ and $c_1,c_2,\eps_0$ as above
concludes the proof of the lemma.

Lastly, for completeness, we show that ${\g_D}\geq \frac{\g_A}{\sqrt{\g_A^2 + d}}$.
\beqn
\nonumber
\g_D^{-1}&=&\nrm{D(\uv,\wv)^{-1}}
\\
\nonumber
&\leq& \sqrt{\nrm{A(\uv,\wv)^{-1}}^2( 1 + \nrm{B(\uv,\wv)}^2) + \nrm{I_s}^2}
\\
\label{eq:C_inv}
&\leq& \sqrt{1 + \frac{1 + \nrm{B(\vv,\wv)}^2}{\g_A^2}}.
\eeqn
To bound $\nrm{B(\uv,\wv)}$, note that $\bm Q(\uv,\wv)$ is linear in $\wv$ and its $i$-th entry is given by
\beq
[\bm Q\colvec{\uv}{\wv}]_i = \sum_{k,l} w_{ikl}u_k u_l - (\sum_{j,k,l} w_{jkl} u_j u_k u_l) u_i.
\eeq
Thus, the ${d \times m}$ matrix $B(\uv,\wv)$ has entries 
\begin{equation}
[B(\uv,\wv)]_{i, (jkl)} = [\nabla_{\wv} \bm Q\colvec{\uv}{\wv}]_{i, (jkl)} 
= (\delta_{ij} - u_i u_j)u_k u_l,
\end{equation}
which is independent of $\wv$.
Recalling that $\nrm{\uv}=1$,
\beq
\nrm{B(\uv)}^2 &\leq& \nrm{B(\uv)}_F^2
= \sum_{i,j,k,l=1}^d (\delta_{ij} - u_i u_j)^2 u_k^2 u_l^2
\\
&=& \sum_{i,j=1}^d (\delta_{ij}^2 - 2 \delta_{ij} u_i u_j + u_i^2 u_j^2)=d - 1.
\eeq
Putting this bound in \eqref{eq:C_inv}, we obtain
${\g_D}\geq \frac{\g_A}{\sqrt{\g_A^2 + d}}$.
\end{proof}

\section{Proof of Lemma \ref{lem:test_vast}}
\label{sec:proof_cumulative_eq}
Let $\vast\in\Vast$.
Then $\exists i\in[d]$ such that ${\vast}^\top W^\top \hv = h_i \in\{0,1\}$.
Hence, by \eqref{eq:mix_all}, the c.d.f.\ $F_{\vast}$ of ${\vast}^\top \fv$ corresponds to the two component GMM 
\begin{equation*}
(1-p_i)\cdot\mathcal{N}(0, \sigma^2\nrm{\vast}^2) 
 + p_i \cdot \mathcal{N}(1, \sigma^2 \nrm{\vast}^2).
\end{equation*}
By Lemma \ref{lem:eig_to_W} we have $p_i = 1/\lambda(\vast)^2$.
Thus, $F_{\vast}=G_{\vast}$.

For the other direction, let $\vv \in \Vsig\setminus\Vast$. 
Since $W$ is full rank,
the $d$-dimensional vector $\uv^\top = \vv^\top W^\top \notin \{\bm e_i^\top\}_{i=1}^d$.
Moreover,
by 
Eq.\ \eqref{eq:separated} of 
Lemma \ref{lem:candidate_stability}, 
\beq
\inf_{\vv\in \Vsig\setminus \Vast} \min_{\vast\in\Vast}\nrm{\vv - \vast} \geq \delta_1 > 0.
\eeq
Hence, there exists $\eps_0>0$ such that 
\beq
\min_{i\in[d]} \nrm{\uv-\bm e_i} \geq \eps_0.
\eeq
So by the expected rigidity condition \eqref{eq:rigid_exp}, there exists $\eta_0>0$ such that $r(\uv)\geq \eta_0$.
It follows that $F_{\vv}$ has a component with mean that is bounded away from both $0$ and $1$ and thus $F_{\vv}\neq G_{\vv}$.
In particular, there exists $\eta_1>0$ such that
\beqn
\label{eq:KS_upper_bound}
\sup_{t\in\R}|F_{\vv}(t)- G_{\vv}(t)|\geq\eta_1.
\eeqn

\section{Proof of Lemma \ref{lem:cont_map_thm}}
\label{sec:proof_cont_map_thm}
Recall that our sample of size $2n$ was split into two separate parts each of size $n$.
The first $n$ samples were used to estimate the tensor eigenvectors, and the last $n$ samples to estimate the empirical cdf's of their projections onto the eigenvectors. 

For any $\hvv$ that is close to a vector $\vv$,
we bound 
$
\score_n(\hvv) = \nrm{\eCom_{\hvv} - G_{\hvv}}_\infty
$
by the triangle inequality,
\begin{equation}
\label{eq:KS_decompose}
\nrm{\eCom_{\hvv} - G_{\hvv}}_\infty
\leq 
\nrm{\eCom_{\hvv} - F_{\hvv}}_\infty +
 \nrm{F_{\hvv} - F_{\vv}}_\infty +
 \nrm{F_{\vv} - G_{\vv}}_\infty +
 \nrm{G_{\vv} - G_{\hvv}}_\infty.
\end{equation}
We now consider each of the four terms separately, starting with the first one.
Since $\sigma>0$, the cdf $F_{\hvv}:\R\to[0,1]$ is continuous and the distribution of
$\nrm{\eCom_{\hvv}-F_{\hvv}}_\infty$ is independent of $\hvv$.
Then, by the Dvoretzky-Kiefer-Wolfowitz inequality, $\nrm{\eCom_{\hvv}-F_{\hvv}}_\infty$ is w.h.p.\ of order $O(1/\sqrt{n})$ for any $\hvv$, and in particular tends to zero as $n\to0$. 

As for the second term, write $\hvv=\vv+\bm \eta$.
Then,
\beq
\hvv^\top\xv=\vv^\top\xv+ \bm\eta^\top\xv.
\eeq
Recall that $\xv=W^\top\hv+\sigma\noisev$. Hence, $|\bm\eta^\top\xv|\leq\nrm{W}_2\sqrt{d}\nrm{\bm\eta}+\sigma|\bm\eta^\top\noisev|$.
The term $\bm\eta^\top\noisev$ is simply a zero mean Gaussian random variable with standard deviation $\sigma\nrm{\bm\eta}$.
So, there exists $K_n>\sqrt{d}\nrm{W}_2 +\sigma n^{1/3}$ such that with probability tending to one as $n\to\infty$, for all $n$ samples $\xv_j\in\samp$, $|\bm\eta^\top\xv_j|\leq K_n\nrm{\bm\eta}$.
Thus, $|\hvv^\top\xv-\vv^\top\xv|$ can be bounded by $K_n\nrm{\hvv-\vv}$.
This, in turn, implies that
\beq
\nrm{F_{\hvv} - F_{\vv}}_\infty \leq L K_n \nrm{\hvv-\vv},
\eeq
where $L=\max_{t} F'_{\vv}(t)$, which is finite for any $\sigma>0$.
Now, suppose the sequence $\hvv_{(n)}$ converges to some $\vv$ at rate $O_P(1/\sqrt{n})$.
Since $K_n$ grows much more slowly with $n$, this term tends to zero.

Let us next consider the fourth term, and leave the third term to the end.
Here note that $G_{\vv}$ is continuous in its parameter $\vv$.
So if the sequence $\hvv_{(n)}$ converges to some $\vv$, then this term tends to zero.

Finally, consider the third term.
If the limiting vector $\vv$ belongs to the correct set, namely $\vast\in\Vast$, then $F_{\vv}=G_{\vv}$, and thus overall $\nrm{\eCom_{\hvv} - G_{\hvv}}_\infty$ tends to zero as required.

In contrast, if $\hvv$ converges to a vector $\vv\notin \Vast$, then instead of Eq.\ \eqref{eq:KS_decompose} we invoke the following inequality:
\begin{equation*}
\nrm{\eCom_{\hvv} - G_{\hvv}}_\infty
\geq 
\nrm{F_{\vv} - G_{\vv}}_\infty -
 \nrm{F_{\vv} - F_{\hvv}}_\infty -
\nrm{F_{\hvv} - \eCom_{\hvv}}_\infty -
\nrm{G_{\hvv} - G_{\vv}}_\infty .
\end{equation*}
Here $\nrm{F_{\vv} - G_{\vv}}_\infty$ is strictly larger than zero whereas the three other remaining terms tend to zero as $n\to\infty$ as above.

\section{Proof of Lemma \ref{lem:part_1}}
\label{sec:proof_eig_equations}
Multiplying \eqref{eq:T_tilde_X} from the right by the full rank matrix $Y^{-1}$ we obtain the equations
\beq
[\wto(Y,Y,I)]_{iij} = [Y^\top]_{ij}
,\quad \forall i,j\in[d].
\eeq
Note that for all $i\in[d]$, 
$$[\wto(Y,Y,I)]_{iij} = [\wto(\yv_i,\yv_i,I)]_j.$$
Since $\wto$ is symmetric, we thus have
\beq
\wto(I,\yv_i,\yv_i) = \yv_i
,\quad \forall i\in[d].
\eeq
Writing $\yv_i = \uv_i/\lambda_i$ we obtain the eigenpair equation
\beq
\wto(I,\uv_i,\uv_i) = \lambda_i\uv_i
,\quad \forall i\in[d].
\eeq
The other direction readily follows from the definition of eigenpairs.

\section{Matrix and tensor denoising}\label{sec:improved_denoising}

In Algorithm \ref{algo:noisy}, we modify the diagonal elements of $M,\M$ by \eqref{eq:moments_sigma}. This modification is suited for additive Gaussian noise, but is not applicable for the case where $X= \text{binomial}(2,W^TH)$. Instead, we implemented a method derived in \cite{jain2014learning} for a similar setup. 

First, we treat the diagonal elements of $M_\sigma$ as missing data, and complete them with the following iterative steps.
(i) compute the first $d$ eigenpairs $\{\vv_i,\lambda_i\}$ of $R^{(k)}$; and (ii) update the diagonal elements by $R^{(k+1)}_{jj} = (\sum_i \lambda_i \vv_i \vv_i^\top)_{jj}$. 

Next, instead of computing $\M_\sigma$ via  \eqref{eq:moments_sigma} and then $\wto_\sigma$ via \eqref{eq:white_tensor_noisy}, we compute $\wto_\sigma$ directly by solving 
the following system of linear equations.
Let $K^\dagger$ be the pseudo-inverse matrix of $K$, and $P_\Omega(\T)$ denote a masking operation
over the tensor $\T$ such that,
\[
P_\Omega(\T) = \begin{cases}
\T_{ijk} & i \neq j \neq k \\
0        & \text{o.w} 
\end{cases} 
\]
We estimate $\wto$ by the following minimization problem,
\[
\hat \wto = \argmin_\wto \| P_\Omega \big( \wto(K^\dagger,K^\dagger,K^\dagger)\big) - P_\Omega(\fto)\|_F^2
\]
This method depends only on the off-diagonal elements of $M$ and $\M$ and hence is applicable whenever $\mathbb E[X|H] = W^T H$ and the noise has bounded variance. 

\section{Adding a weighted least square step to the spectral method}\label{sec:weighted_ls}
In section \ref{sec:experiments}, we compare the results of algorithm \ref{algo:noisy} with and without an additional single weighted least square step.
Given an estimate $\hat W$, for each observed instance $\bm x_j$ we calculate the conditional likelihood $\mathcal L(\bm x_j|\bm h)$ for the $2^d$ possible binary vectors $\bm h \in \{0,1\}^d$,
\[
\mathcal L(\bm x_j|\bm h) = \frac{1}{\sqrt{2\pi\sigma^2}}\exp\Big(-\|\bm x_j-\hat W^T \bm h\|^2\Big/(2\sigma^2)\Big)
\]
For each instance $\bm x_j$, we keep the top $K=6$ vectors $\bm h_{1j},\ldots,\bm h_{Kj}$ with the highest likelihood. Let $\Pi \in [0,1]^{K \times n}$ be a weight matrix such that 
$\Pi_{kj}$ is proportional to $\mathcal L(\bm x_j|\bm h_{kj})$, and $\sum_k \Pi_{kj} =1$ for all $j$. The new estimate $\hat W_{\text{wls}}$ is the minimizer of the weighted least square problem,
\[
\hat W_{\text{wls}} = \argmin_W \sum_{j = 1}^n \sum_{k=1}^K \Pi_{kj} \|\bm x_j-W^T \bm h_{kj}\|^2.
\]

\section{Alternating least squares for $W$ and $H$}\label{sec:als_method}
In section \ref{sec:experiments}, we compare the results of the spectral approach to the following ALS iterations, with a random starting point.
\begin{align}
W^{(k)} & = \argmin_{W \in \mathbb R^{d \times m}} \|X-W^T H^{(k-1)}\|_F^2 \notag \\
\hat H^{(k)} & = \argmin_{H \in \mathbb R^{d \times n}} \|X-(W^{(k)})^T H\|_F^2 \notag \\
H^{(k)} & = \argmin_{H \in \{0,1\}^{d \times n}}\|H-\hat H^{(k)}\|_F^2 \notag, \notag 
\end{align}

\section{Genetic admixture simulations}
\label{sec:genetic_admixture}
The simulated admixture data was generated via the following steps:
\begin{enumerate}
\item We used SCRM \citep{scrm} to simulate a split between $d=3$ ancestral populations, with separation time of $4000$ generations. The simulator generated $40$ chromosomes of length $250 \cdot 10^6$ for each of the three populations.
The simulation parameters were determined as $N_0=10^4$ effective population size, $10^{-8}$ mutation rate (per base pair per generation), and $10^{-8}$ recombination rate (per base pair per generation).

\item We sampled the proportion matrix $W$ from a Dirichlet distribution with parameter $\alpha$. 
\item 
Two chromosomes of length $250 \cdot 10^6$ were created for each of the $m = 50$ admixed individuals with the following steps: (i) An ancestral population was sampled according to W, say, population $h_A$. (ii) One of the $40$ chromosomes was sampled from $h_A$, say $h_A(k)$ (iii) A block length $l$ was sampled from an exponential distribution with rate $20$ per Morgan corresponding admixture event happening $20$ generations ago (in our case, 1 Morgan was $10^8$ base pairs). (iv) A block of length $l$ was copied from chromosome $h_A(k)$ to the corresponding locations in the new admixed chromosome. We repeated steps (i)-(iv) until completion of the chromosome.

\end{enumerate}

\end{document}